\documentclass[twoside,11pt]{article}

\usepackage{blindtext}

\usepackage{jmlr2e}
\usepackage{amsmath}
\usepackage{amssymb}
\usepackage{mathtools}
\usepackage{xcolor}

\usepackage{lastpage}
\jmlrheading{23}{2026}{1-\pageref{LastPage}}{1/21; Revised 5/22}{9/22}{21-0000}{Duy Hoang, Bastien Berret, Olivier Bruneau, and Laurent Fribourg}

\ShortHeadings{A Data-dependent Early Stopping Rule using Rademacher Complexity with $L_1$-norm}{Hoang, Berret, Bruneau, and Fribourg}
\firstpageno{1}

\begin{document}

\title{A Data-dependent Early Stopping Rule using Rademacher Complexity with $L_1$-norm  }

\author{\name Duy Hoang \email hoangduy@lmf.cnrs.fr \\
       \addr Université Paris-Saclay, CNRS, ENS Paris-Saclay, LMF \\
       Gif-sur-Yvette, 91190, France
       \AND
       \name Bastien Berret \email bastien.berret@universite-paris-saclay.fr \\
       \addr Université Paris-Saclay, Inria, CIAMS\\
       Gif-sur-Yvette, 91190, France
       \AND
       \name Olivier Bruneau \email olivier.bruneau@ens-paris-saclay.fr \\
       \addr Université Paris-Saclay, ENS Paris-Saclay, LURPA\\
       Gif-sur-Yvette, 91190, France
       \AND
       \name Laurent Fribourg \email fribourg@lmf.cnrs.fr \\
       \addr Université Paris-Saclay, CNRS, ENS Paris-Saclay, LMF\\
       Gif-sur-Yvette, 91190, France}


\maketitle

\begin{abstract}Training neural networks requires balancing the trade-off between fitting the training data and achieving robust performance on unseen inputs. This ability, commonly referred to as generalizability, is determined by the gap between the empirical risk on the training set (``empirical loss”) and the expected risk over the data distribution (``generalization error”). Existing approaches typically estimate the generalization error numerically, requiring gradient descent training and an ``early stopping'' strategy. 
In this work, we introduce an analytic framework that estimates the optimal time of early stopping without the need for training. Several works in the literature also give such analytical estimations, but they are generally based on random matrix theory and often make assumptions on the distribution of the data or the eigenvalue distribution of the covariance matrix. In contrast, our work is based on Rademacher complexity (RC) without needing such probabilistic assumptions. 
For both  theoretical and numerical reasons, it is more relevant to express RC  with the $L_1$-norm rather than with the $L_2$-norm.
We focus on the case of linear models and the problem of linear regression. Thanks to the “linear probing” method, our results can, however, be successfully applied to nonlinear neural networks, as illustrated in the classification MNIST example.
\end{abstract}


\section{Introduction}\label{sec:intro}
Given a distribution $\cal{D}$ of input/output pairs, the expected loss across $\cal{D}$ is referred to as a ``population loss'', and denoted~$L_{{\cal D}}$. 
Using a neural network  (NN) and a process of gradient descent (GD), one can
estimate $L_{{\cal D}}$ on
a set~$S$ of $n$ samples randomly
selected from~$\cal{D}$.
This is called the ``empirical loss'', and denoted $L_S$.
The generalization loss, denoted $L_G$, is the difference between  $L_{\cal{D}}$ and $L_S$,
and accounts for the data located outside~$S$. At the beginning of the GD process, $L_S$ tends to decrease while $L_G$ tends to increase (``bias-variance'' tradeoff), so
the {\em early stopping} strategy seeks to halt GD at a time $t^*$ minimizing $L_S+L_G$.
Classically, $t^*$ is estimated {\em numerically} 
 as the time  $t_{test}$ that minimizes the empirical loss
on a separate dataset~$S_{test}$ (see, e.g., \cite{prechelt2002}).
Using {\em random matrix theory} (RMT), several works estimate~$t^*$
{\em analytically}, without needing to perform GD
(e.g., \cite{raskutti14,LiaoC18,ali2019,AdvaniSS20}).
However these works  often assume the distribution ${\cal D}$ to be  Gaussian, or the eigenvalue distribution of the data covariance matrix to be Marchenko–Pastur (\cite{le1991eigenvalues}).
In contrast here, we estimate $t^*$ analytically using {\em Rademacher complexity} (RC) theory (\cite{BartlettM02}), and do not make any assumption about the data distribution. 
Our method is based on  a data-dependent criterion of the form ${\cal C}(s)$ 
that ensures
$L_S(t)+L_G(t)$ to {\em decrease} for all $t\leq s$.  
We show that the highest $s$ satisfying ${\cal C}(s)$  is  a low estimate~$t^+$ of~$t^*$ 
(see Proposition~\ref{prop:1}). We also give sufficient conditions that ensure $t^+=t^*$.
In order to eliminate a factor~$M$ dependent on  ${\cal D}$, 
we estimate RC using the $L_1$-norm rather than the $L_2$-norm 
(see Remark~\ref{rk:L1}).
On the examples (see Section~\ref{sec:ex}), we check  that  the $L_1$-norm yields a value of $t^+$ much closer to the numerical stopping time~$t_{test}$ than the $L_2$-norm.

We also give an analytic form for the value of $L_S+L_G$ at $t=\infty$
(see Section~\ref{sec:benefit}). By comparing it with $L_S+L_G$ at~$t=t^+$, we determine whether the early stopping strategy should be applied or not
(see, e.g., \cite{sonthalia2024,Bartlett-benign,Belkin18}).
Our work focuses on  {\em linear} models. 
However, our results can be applied to {\em nonlinear} NNs thanks to the linear probing strategy
(see \cite{alain2016}).
%
This is illustrated on MNIST classification examples (Examples~\ref{ex:1} and~\ref{ex:2}, Section~\ref{sec:ex}).
The proofs of our results are given in Appendix.

%

%
%
\subsection*{Comparison with related work}
As mentioned earlier,  several works obtain analytical upper bounds on the 
generalization loss using RMT,  and used them
to estimate  the optimal stopping time.
%
More precisely, let $\lambda_1\geq \lambda_2\geq\cdots\geq \lambda_n\geq 0$ be
the eigenvalues of the
data covariance matrix.
In \cite{AdvaniSS20} for example, they evaluate a stopping time  that minimizes the average generalization dynamics,
by determining the error associated with each mode $i$ ($i\in[n]$). For mode $i$, they find an optimal stopping time of the form
\begin{equation*} 
t^{\text{opt}}=\frac{1}{\lambda_i}\ln(1+\lambda_i \cdot SNR),
\end{equation*}
where $SNR$ is a signal-to-noise ratio.
Here,  without making any probabilistic assumptions, we find a similar result using Rademacher complexity, viz.,
an estimate of the optimal stopping time of the form
(see Equation~\eqref{eq:tapprox})
\begin{equation*} 
t^+\approx \frac{1}{\lambda_1}\ln \frac{\Gamma(0)}{\Omega(0)}.
\end{equation*} 
The numerator $\Gamma$ depends on the higher eigenvalues
$\lambda_1,\dots,\lambda_\alpha$, and contains the ``informative'' part (see, e.g., \cite{OymakF2019}).
The denominator $\Omega$ depends on the lower eigenvalues
$\lambda_{\alpha+1},\dots,\lambda_n$, and contains the ``nuisance'' part. So the quotient $\Gamma/\Omega$ can be interpreted as a form of SNR ratio.


In the literature, Rademacher complexity (together with
``Neural Tangent Kernel'' theory) has often been used 
to find analytical bounds
on the population loss: see, e.g.,
\cite{JacotHG18,Du,AroraDHLW19,allen2019,OymakF2019,LiSO20}.
%
However,  these studies  require 
the number~$m$ of NN parameters be  large 
compared to the number~$n$ of samples  ({\em overparameterization}). 
In this case, a phenomenon of ``benign overfitting'' or ``epochwise double descent'' appears (see, e.g., \cite{heckel2020early,stephenson2021,Nakkiran_2021}): The  loss
$L_{test}(t)$   reaches a first local minimum at $t=t_{test}$, then increases before going down later, converging towards a minimum lower than at $t=t_{test}$. In this context, the 
strategy of early stopping is not ``beneficial''. In contrast,
our method is well adapted  to the {\em underparameterized} case 
(i.e., $m\leq n$). 
%
%
%
%
%
\subsection*{Notation}
In this paper, $\mathbb{R}$ and $\mathbb{N}$ refer to the sets of real and natural numbers, respectively. 
We denote by $\mathbb{R}^p$ a $p$-dimensional Euclidean space, and by $\mathbb{R}^{p\times q}$ a space of real matrices with $p$ rows and $q$ columns. 
We use bold letters for vectors and bold capital
letters for matrices. 
For a given matrix $\boldsymbol{M} \in \mathbb{R}^{p \times q}$, we use 
$\boldsymbol{M}^\top$ for its transpose,
and $\boldsymbol{M}^\dagger$ for its Moore-Penrose (pseudo)inverse. 
The $L_p\text{-norm}$ of a vector $\boldsymbol{v}$ is denoted by $\|\boldsymbol{v}\|_p$.
%
%
We  use {\em i.i.d.} to indicate the set of independent and identically distributed random variables. 
We use~$[n]$ for $\{1,\dots,n\}$, and $\boldsymbol{I}$
for the identity matrix.
For $v\in\mathbb{R}$, we use $\text{sgn}(v)$ to denote $1$ if $v\geq 0$, or $-1$ if $v<0$. For $\boldsymbol{v}=(v_1,\dots,v_n)\in\mathbb{R}^n$,
we use $\text{sgn}(\boldsymbol{v})$ for $(\text{sgn}(v_1),\dots,\text{sgn}(v_n))$.

\section{Preliminary Results}\label{sec:prelim}

%
We consider a distribution $\cal{D}$ over ${\cal H}\times \cal{Y}$ where 
${\cal H}\subset$ $\mathbb{R}^{m}$
is the input space of all possible instances $\boldsymbol{h}$,
and $\cal{Y}\subset \mathbb{R}$ the space of the corresponding outputs.

\subsection{Population loss and Rademacher complexity}\label{ss:rademacher}
The ``training set'' $S$  is a set of $n$ input/output pairs 
$\{(\boldsymbol{h}_1,y_1),\dots,(\boldsymbol{h}_n,y_n)\}$ made of~$n$ samples selected {\em i.i.d.}
from $\cal{D}$. Let $\boldsymbol{y}=(y_1,\dots,y_n)\in\mathbb{R}^n$.
Let $\boldsymbol{K}$
be the $n\times m$ matrix having 
$\boldsymbol{h}_1,\dots,\boldsymbol{h}_n\in\mathbb{R}^m$ as columns.
The data covariance matrix $\boldsymbol{H}\in\mathbb{R}^{n\times n}$ is defined by 
$$\boldsymbol{H}=\boldsymbol{K} \boldsymbol{K}^\top,$$
i.e., the $(i,j)$-entry of $\boldsymbol{H}$ is $H_{i,j} = \boldsymbol{h}_i^\top \boldsymbol{h}_j$. (See, e.g., \cite{Du}.)
%
As in~\cite{XavierCF25}, we focus here on the $L_1$-norm
(see Remark~\ref{rk:L1}).
In this context,
the {\em population loss} $L_{\cal D}[f]$ (more simply denoted
as $L_{\cal{D}}$) over data distribution~${\cal D}$  is: 
$$L_{{\cal D}}[f]=\mathbb{E}_{(\boldsymbol{h},y)\sim {\cal D}}[|f(\boldsymbol{h})-y|],$$
where $f:\mathbb{R}^m\rightarrow \mathbb{R}$ is a given function. As we focus here on the linear regression problem,  $f$ is of the form:
$$f(\boldsymbol{h})=\boldsymbol{a}^\top \boldsymbol{h}$$
where $\boldsymbol{a}=(a_1,\dots,a_m)\in \mathbb{R}^m$.
The {\em empirical loss} $L_S$ 
over~$S$ is defined by:
\begin{equation}\label{eq:LS}
L_S=\frac{1}{n}\sum_{i\in[n]}|\boldsymbol{a}^\top \boldsymbol{h}_i-y_i|
=\frac{1}{n}\sum_{i\in[n]}|v_i|=\frac{1}{n}\|\boldsymbol{v}\|_1
\end{equation}
with
\begin{align}
\label{eq:vi}
v_i=&\;\boldsymbol{a}^\top \boldsymbol{h}_i -y_i\in\mathbb{R}\ \ \ (i\in[n]),\\
\label{eq:vv}
\boldsymbol{v}=&\;(v_1,\dots,v_n)=\boldsymbol{K}\boldsymbol{a}-\boldsymbol{y}\in\mathbb{R}^n.
\end{align}
The vector $\boldsymbol{v}$ is called the {\em empirical error vector} (or the {\em training error vector}).
%
We are searching for a vector $\boldsymbol{a}$ that minimizes the population loss $L_{{\cal D}}$ across the entire data distribution~$\cal {D} $. Since $\cal D$ is unknown, our objective is actually to minimize an {\em upper bound} on~$L_{{\cal D}}$. 
Let $M$ and $C$ be two positive reals satisfying respectively:
\begin{equation}\label{eq:M}
|\boldsymbol{a}^\top\boldsymbol{h} - y|\leq  M\ \ \  \forall (\boldsymbol{h},y)\in{\cal H}\times{\cal Y}\\
\end{equation}
\begin{equation}
\label{eq:C2M}\ \ \ 
\mathbb{E}_{\boldsymbol{h}\sim {\cal D}}\left [\|\boldsymbol{h}\|_2^2 \right ]\leq C^2.
\end{equation}
Let $L_G^*$ and $L_{{\cal D}}^*$ be defined as:

\begin{align}
\label{eq:LG}
L_G^*=&\;\frac{2\|\boldsymbol{a}\|_2C}{\sqrt{n}}\\
\label{eq:L_D*}
L_{{\cal D}}^*=&\; L_S + L_G^* 
= \frac{1}{n}
\|\boldsymbol{v}\|_1 + \frac{2\|\boldsymbol{a}\|_2C}{\sqrt{n}}.
\end{align}

We follow 
an approach based on RC
(see \cite{XavierCF25}).
In the linear setting, the result  is as follows.
\begin{proposition}\label{prop:Rademacher L1}
(cf. Proposition 3 of \cite{XavierCF25})
With probability at least $1-\delta$ over the sample~$S$ of size $n$, 
the population loss $L_{{\cal D}}$ satisfies: 
\begin{equation} \label{eq:LD0}
L_{{\cal D}} 
\leq  L_{{\cal D}}^* + \epsilon
\end{equation}
where
\begin{equation}\label{eq:e}
\epsilon=3M\sqrt{\frac{\log \frac{2}{\delta}}{2n}}.
\end{equation}
\end{proposition}

\begin{remark}\label{rk:L1} (justification of the use of  $L_1$-norm)
By Theorem 11.3 of \cite{mohri2018foundations} (p.270), we have
$$\mathbb{E}_{(\boldsymbol{h},y)\sim{\cal D}}[|\boldsymbol{a}^\top\boldsymbol{h}-y|^p]
\leq \frac{1}{n}\|\boldsymbol{v}\|_p + 2\mu_p {\cal R}_n+\epsilon_p$$
where, for linear  models, ${\cal R}_n\leq \|\boldsymbol{a}\|_2C/\sqrt{n}$ (see \cite{ma2022}, Theorem 5.10),
$\mu_p=pM^{p-1}$ and $\epsilon_p=3M^p\sqrt{\log(2/\delta)/2n}$. Hence:
$$\mathbb{E}_{(\boldsymbol{h},y)\sim{\cal D}}[|\boldsymbol{a}^\top\boldsymbol{h}-y|^p]-\epsilon_p
\leq \frac{1}{n}\|\boldsymbol{v}\|_p + pM^{p-1} \frac{2C\|\boldsymbol{a}\|_2}{\sqrt{n}}.$$
In order to {\em eliminate} $M$ in the right-hand side, we let $p=1$. 
We get:
$$L_{{\cal D}}-\epsilon=\mathbb{E}_{(\boldsymbol{h},y)\sim{\cal D}}[|\boldsymbol{a}^\top\boldsymbol{h}-y|]-\epsilon
\leq \frac{1}{n}\|\boldsymbol{v}\|_1 + \frac{2C\|\boldsymbol{a}\|_2}{\sqrt{n}}=L_{{\cal D}}^*,$$
which corresponds to~\eqref{eq:LD0}. 
The use of the $L_1$-norm allows us to define $L_{{\cal D}}^*$ without using~$M$. The time~$t^+$ at which
the gradient flow is {\em stopped} 
(see Equation~\eqref{eq:t+})
corresponds to
the minimum of~$L_{{\cal D}}^*$, and does {\em not} depend on~$M$ either. The constant $M$ is only used in the definition of $\epsilon$ (see~\eqref{eq:e}). On the examples (see Section~\ref{sec:ex}),
we verify that the use of the $L_1$-norm yields better results than the use of the $L_2$-norm.
Note that the $L_1$-norm is used to evaluate the Rademacher complexity, but it is the classical {\em quadratic}  loss ${\cal L}(\boldsymbol{a})=\frac{1}{2}\|\boldsymbol{v}\|_2^2$ (see~\eqref{eq:L0})
that is used by the
gradient flow (see~\eqref{eq:da}).
\end{remark}
\begin{remark}\label{rk:C} (estimates of $C$ and $M$)
As an upper bound
$C$ satisfying~\eqref{eq:C2M}, we can take:\\ 
$C=\sup_{\boldsymbol{h}\in{\cal H}}\|\boldsymbol{h}\|_2.$
As an upper bound $M$ satisfying~\eqref{eq:M}, we can take
$$
M=||\boldsymbol{a}(t^+)||_2 \times\sup_{\boldsymbol{h}\in{\cal H}}\|\boldsymbol{h}\|_2+\sup_{y\in{\cal Y}}|y|.$$
(This is possible because the definition \eqref{eq:t+} of $t^+$ does not involve $M$.)
Note that these estimates of $C$ and $M$  suppose that the domains  ${\cal H}$ and ${\cal Y}$ are  {\em bounded}.
Note also that, for the problem of binary classification, we have $M=2$.
\end{remark}
%
\subsection{Gradient flow for linear models}\label{ss:gf}
For simplicity, we express the problem in the continuous-time setting, and consider the gradient flow process (GF) instead of discrete-time GD.
The idea  is to  apply~GF
to $\boldsymbol{a}$, and stop the process at the time $t^*$
where $L_{{\cal D}}^*$ is expected to reach its minimum (strategy of ``early stopping''). Classically, one estimates $t^*$ by
considering a separate set of data $S_{test}$, and determine
the time $t_{test}$ where GF reaches its minimum on that set (see,e.g., \cite{prechelt2002}).
We give here a method that allows us to estimate $t^*$ analytically without needing a separate set or having to apply~GF.
%
More formally, we consider the problem of finding $\boldsymbol{a}=(a_1,...,a_m)\in\mathbb{R}^m$ that minimizes the following quadratic loss function:
\begin{equation} \label{eq:L0}
{\cal L}(\boldsymbol{a})= \frac{1}{2}\sum_{i=1}^n (\boldsymbol{a}^\top \boldsymbol{h}_i-\boldsymbol{y}_i)^2
=\frac{1}{2}\sum_{i=1}^n v_i^2
=\frac{1}{2}\|\boldsymbol{v}\|_2^2.
\end{equation}
%
The vector~$\boldsymbol{a}$ that  minimizes 
${\cal L}(\boldsymbol{a})$ on ${\cal D}$ is found by considering a given training set $S$ containing $n$ samples $\left(\boldsymbol{h}_i, y_i\right)$ drawn randomly i.i.d. from $\cal{D}$.
The vector $\boldsymbol{a}$ is initialized to 0,
and is updated via GF on $S$  as follows:
 $$\frac{d \boldsymbol{a}}{dt} = -\frac{\partial {\cal L}}{\partial \boldsymbol{a}}.$$
Hence, using \eqref{eq:vi}, \eqref{eq:vv} and \eqref{eq:L0}, we have:
\begin{equation}\label{eq:da}
\begin{aligned}
     \frac{d \boldsymbol{a}}{dt} =&\; -\frac{\partial {\cal L}}{\partial \boldsymbol{a}}
    =-\frac{1}{2}\sum_{i=1}^n \frac{\partial v_i^2}{\partial \boldsymbol{a}}
    = -\sum_{i=1}^nv_i\boldsymbol{h}_i = -\boldsymbol{K}^\top\boldsymbol{v}.
\end{aligned}
\end{equation}
On the other hand, the dynamic of the training error vector~$\boldsymbol{v}$ during GF is expressed by (see, e.g., \cite{Du}):
\begin{equation}\label{eq: v dynamic ln}
    \frac{d \boldsymbol{v}}{dt} = - \boldsymbol{H}\boldsymbol{v}.
\end{equation}
The GF process makes the first term $\|\boldsymbol{v}\|_1/n$ of $L_{{\cal D}}^*$
(see~\eqref{eq:L_D*})
decrease, and the second term 
$2\|\boldsymbol{a}\|_2C/\sqrt{n}$
increase. The curve $L_{{\cal D}}^*$ is thus
``U-shaped'' \footnote{At least in a first phase, since the curve may decline later on (``epoch-wise double descent'').}. It reaches a first local minimum at~$t=t^*$,
which is the first time when $dL_{{\cal D}}^*(t)/dt\geq 0$:
\begin{equation}\label{eq:t**}
t^*=\inf_{t\geq 0}\{t: \frac{dL_{{\cal D}}^*(t)}{dt}\geq 0\}.
\end{equation}
\begin{remark}\label{rk:init}
The  initialization of $\boldsymbol{a}$ to $0$ has the effect of limiting the growth of the norm $\|\boldsymbol{a}(t)\|_2$   and  the generalization loss $L_G^*=2C\|\boldsymbol{a}(t)\|_2/\sqrt{n}$ during GF.
\end{remark}
\section{A Data-dependent Estimate $t^+$ of $t^*$}\label{sec:III} 
We now explain how to estimate the first local minimum of~$L_{{\cal D}}^*$ without needing to apply~GF.
\subsection{Analytic form of $dL_{{\cal D}}^*/dt$}


\begin{proposition}
\label{prop: 3.1}
The derivative of $L_{{\cal D}}^*$ is given by:
\begin{align}\label{eq: 67}
\frac{dL_{{\cal D}}^*(t)}{dt} =&\; -\frac{1}{n}\left(\text{sgn}(\boldsymbol{v}(t))\right)^\top\boldsymbol{H}\boldsymbol{v}(t)+  \Psi(t) \\
\intertext{with}
\label{eq:13bis}
\Psi(t)=&\;-\frac{2C}{\sqrt{n}}\frac{(\boldsymbol{K}^\dagger
(\boldsymbol{v}(t) + \boldsymbol{y}))^\top}
{\|\boldsymbol{K}^\dagger\left(\boldsymbol{v}(t) + \boldsymbol{y}\right)\|_2}
\boldsymbol{K}^\top \boldsymbol{v}(t)\in\mathbb{R},
\end{align}
where $\boldsymbol{K}^\dagger\in\mathbb{R}^{m\times n}$ is the pseudoinverse of~$\boldsymbol{K}$.
\end{proposition}

\subsection{Identification of an area where~$L_{{\cal D}}^*$ decreases}
%
%
We now consider the set ${\cal V}$ of eigenvalues of $\boldsymbol{H}$:
$\lambda_1\geq\lambda_2\geq\cdots\lambda_n\geq 0$.
This set  typically breaks down into  a
set ${\cal{V}}_1=\{\lambda_1,\dots,\lambda_{\alpha}\}$ made of
a small number of large values, and the remaining
set  ${\cal{V}}_2=\{\lambda_{\alpha+1},\dots,\lambda_{n}\}$
made of a ``bulk'' of low values.  See, e.g., \cite{OymakF2019,ghorbani2019,AdvaniSS20,MurrayJBM23}.


 
Let $\boldsymbol{P} \in \mathbb{R}^{n\times n}$ the transition matrix satisfying
$$\boldsymbol{H}=\boldsymbol{P}\text{diag}(\lambda_1,\dots,\lambda_{n})\boldsymbol{P}^\top$$
and $\boldsymbol{P}_i$ the $i$-th column of $\boldsymbol{P}$.
Let $\boldsymbol{u}_0 = \boldsymbol{P}^\top\boldsymbol{v}(0) \in \mathbb{R}^{n}$ and $U_i\in\mathbb{R}$ be the $i$-th component of $\boldsymbol{u}_0$. 
%
We have:
\begin{equation}\label{eq:v}
\boldsymbol{v}(t)=\sum_{i=1}^n \boldsymbol{w}_ie^{-\lambda_i t}
\end{equation}
with $\boldsymbol{w}_i=U_i\boldsymbol{P}_i$.
Let
\begin{align*}
    \Gamma_i(t)=&\;\left(\text{sgn}(\boldsymbol{v}(t))\right)^\top\boldsymbol{H}\boldsymbol{w}_i\ \ (i\in[\alpha]),\\
    \Gamma(t)=&\;\sum_{i\in[\alpha]}\Gamma_i(t),\\
    \Delta(t)=&\;\left(\text{sgn}(\boldsymbol{v}(t))\right)^\top\boldsymbol{H}\sum_{j=\alpha+1}^n \boldsymbol{w}_{j} e^{-\lambda_jt},\\
    \Omega(t)=&\; n\Psi(t)-\Delta(t).
\end{align*}

\begin{remark}
Note that $\Gamma$ involves the higher eigenvalues $\lambda_i$ ($i\in[\alpha]$),
and $\Delta$, hence $\Omega$,  the lower ones. As the higher eigenvalues are related to the
``informative'' space, and the lower ones to the ``nuisance'' one (see \cite{OymakF2019}), 
the quotient $\Gamma/\Omega$ can be interpreted
as a form of signal-to-noise ratio (SNR).
\end{remark}

We decompose the time space into contiguous time intervals
${\cal T}_1$, ${\cal T}_2,\dots$ over which
the signs of
each $\Gamma_i$ ($i\in[\alpha]$) and $\Omega$ are invariant (``sign-invariance''). So over each time interval~${\cal T}$ and each $i\in[\alpha]$, we have:
\begin{itemize}   
\item $\Gamma_i(t)> 0\ \forall t\in{\cal T}$
\ \ or\ \ \ $\Gamma_i(t)\leq 0 \ \forall t\in{\cal T}$,\ \ and
\item $\Omega(t)>0\ \forall t\in{\cal T}$
\ \ or\ \  $\Omega(t)\leq 0\ \forall t\in{\cal T}$.
\end{itemize}
We define
\begin{align*}
I_+({\cal T})=&\;\{i\in[\alpha]: \Gamma_i(t)> 0\ \forall t\in{\cal T}\}\\
I_-({\cal T})=&\;\{i\in[\alpha]: \Gamma_i(t)\leq 0\ \forall t\in{\cal T}\}.
\end{align*}
Let
\begin{equation}\label{eq:CS3}
\Phi(t)=
\sum_{i\in[\alpha]}\Gamma_i(t) e^{-\lambda_i t}  -\Omega(t).
\end{equation}
Note that the only time-varying terms 
of $\Phi(t)$ are  $\boldsymbol{v}(t)$ and~$e^{-\lambda_i t}$
($i\in[n]$).
 Since $\boldsymbol{v}(t)$ is itself a linear combination of $e^{-\lambda_i t}$ (see~\eqref{eq:v}),  the time-varying terms of~$\Phi(t)$ are just (products of)  $e^{-\lambda_i t}$.
Let us consider interval ${\cal T}_1$ (assumed to be of the form $[0,\tau_1)$, and consider the time $t^+$ defined, using $\Phi(t)$, as:
\begin{equation}\label{eq:t+}
t^+=\sup_{s\in {\cal T}_1}\left\{s: \Phi(t)> 0,\ \forall t\in[0,s)
\right\}.
\end{equation}
%
%
We suppose $\Phi(0)>0$. (Otherwise, $t^*=0$ and our method fails.)
We will show 
that $t^+$ is such that:
$\frac{dL_{{\cal D}}^*(s)}{dt}<0$ for all $s\in[0, t^+)$
(see~\eqref{eq:decrease0}). Hence, $L_{{\cal D}}^*$
is decreasing on $[0,t^+)$.
We can thus take
$t^+$ as a low estimate of $t^*$  
(which is the first time $s$ such that $dL_{{\cal D}}(s)/dt\geq 0$).
We have $t^+\leq t^*$ and, under certain conditions: $t^+=t^*$ (case~1 of Proposition~\ref{prop:1}).
 Formally:
%
\begin{proposition}\label{prop:1} 
We have:
\begin{equation}\label{eq:basic00}
    \frac{dL_{{\cal D}}^*(t)}{dt}< 0\ \ 
    \ \ \text{ iff }\ \ \Phi(t)> 0,\ 
\end{equation}
\begin{align}\label{eq:decrease0}
t^+=&\; \sup_{s\in {\cal T}_1}\left\{s: \frac{dL_{{\cal D}}^*(s)}{dt}< 0\,\ \forall t\in[0,s)\right\},\\
\label{eq:21bis}
t^+\leq&\; t^*.
\end{align}
There are two cases: 
\begin{itemize}
\item if $t^*\in{\cal T}_1$ (case 1), we have
\begin{align}\label{eq:case1}
t^+=t^*<&\; \tau_1. &
\end{align}
\item if $t^*\not\in{\cal T}_1$ (case 2), we have
\begin{align}\label{eq:case2}
t^+=\tau_1\leq&\; t^*. &
\end{align}
\end{itemize}
Suppose furthermore: 
\begin{equation}\label{eq:spec}
I_-({\cal T}_1)=\emptyset\ \mbox{ and }\ \ \Gamma(0)>\Omega(0)>0.
\end{equation}
Then:
\begin{align}\label{eq:encadrement}
t_1^+\leq t^+\leq&\; t_\alpha^+ &
\end{align}
where, for $i\in\{1,\alpha\}$, $t_i^+$ is defined as:
\begin{equation}\label{eq:newt+}
\begin{aligned}
t_i^+ &=\sup_{s\in {\cal T}_1} \left\{s: t<\frac{1}{\lambda_i}\ln \frac{\Gamma(t)}{\Omega(t)} \ \ 
\forall t\in[0,s)\right\}.
\end{aligned}
\end{equation}
\end{proposition}

    \begin{remark} \label{rk:t+} Equation \eqref{eq:t+} 
     tells us that $t^+$  is the  least solution of 
     equation $\Phi(t)=0$, where~$\Phi(t)$
     involves only time-varying terms of the form $e^{-\lambda_j t}$ ($j\in[n]$).
        One can thus calculate $t^+$ using classical numerical methods (e.g., fixed-point iteration, Newton method) without needing to apply~GF.
        \end{remark}
    \begin{remark} Under \eqref{eq:spec}, Equation \eqref{eq:encadrement} 
    of Proposition~\ref{prop:1} tells us that
        $t^+$, as defined by~\eqref{eq:t+}, is between
        $t_1^+$ and $t_\alpha^+$ where $t_i^+$ ($i\in\{1,\alpha\}$) is the 
        least solution of the equation:
        \begin{equation*} 
        t=\frac{1}{\lambda_i}\ln\frac{\Gamma(t)}{\Omega(t)}.
        \end{equation*}
        Here again, $t_1^+$  and $t_\alpha^+$ can be computed with classical numerical methods without needing to apply~GF.
        We can even  avoid the use of numerical methods 
        by simply approximating~$t^+$ by
        \begin{equation}\label{eq:tapprox}
        t^{+}_{\text{approx}}= \frac{1}{\lambda_1}
        \ln\frac{\Gamma(0)}{\Omega(0)}.
        \end{equation} 
        As shown in Examples~\ref{ex:1} and~\ref{ex:2} of Section~\ref{sec:ex}, 
        the values of $L_{{\cal D}}^*(t)$ at $t=t^+$ and
        $t=t^+_{\text{approx}}$ 
        are often almost the same (see Tables 1 and 2).

    \end{remark}
\subsection{Beneficial early stopping}\label{sec:benefit}
The first local minimum reached by $L_{{\cal D}}^*(t)$ at $t=t^*$
corresponds to the minimum of the ``U-shaped'' phase 
of $L_{{\cal D}}^*=L_S+L_G^*$ (where $L_S$ decreases while $L_G^*$ increases). Later, the curve $L_{{\cal D}}^*$ can
fall again, and pass through other local minima. This is related to the phenomenon of ``epoch-wise double descent'' 
(see, e.g., \cite{heckel2020early,stephenson2021,Nakkiran_2021}). It is thus interesting to compare the value of  $L_{{\cal D}}^*(t)$ at $t=t^*$ or at  $t=t^+$ (as given by \eqref{eq:t+})
with 
$L_{{\cal D}}^*(\infty)$. If $L_{{\cal D}}^*(t^+)< L_{{\cal D}}^*(\infty)$,
the early stopping strategy is said to be ``beneficial''; otherwise,  a phenomenon of ``benign overfitting'' may occur, and it may be interesting to continue the training after $t^+$
(see, e.g., \cite{Bartlett-benign,sonthalia2024}). 
We now explain
how to obtain an analytic form for
$L_{{\cal D}}^*(\infty)$  (see \eqref{eq:infty}). 
Since $a(t)$ has been initialized to 0, we can show using \eqref{eq:vv} that at $t=\infty$ (cf.  \cite{bjorck1973}):
\begin{align*}
    \boldsymbol{a}(\infty)=&\;\boldsymbol{K}^\dagger \boldsymbol{y},\\
    \boldsymbol{v}(\infty)=&\;(\boldsymbol{K}\boldsymbol{K}^\dagger -\boldsymbol{I})\boldsymbol{y}.
\end{align*}
%
At $t=\infty$, Equation \eqref{eq:L_D*} gives:
\begin{align}
L_{{\cal D}}^*(\infty) =&\;  \frac{1}{n}
\|\boldsymbol{v}(\infty)\|_1 + \frac{2C\|\boldsymbol{a}(\infty)\|_2}{\sqrt{n}},
\intertext{hence:}
\label{eq:infty}
L_{{\cal D}}^*(\infty)=&\: \frac{1}{n}
\|(\boldsymbol{K}\boldsymbol{K^\dagger} - \boldsymbol{I})\boldsymbol{y}\|_1 + \frac{2C\|\boldsymbol{K^\dagger} \boldsymbol{y}\|_2}{\sqrt{n}}.
\end{align}
Note that the epochwise double descent typically happens in the overparametrization case (i.e., $m\gg n$). 
\section{Examples}\label{sec:ex}
We consider a training set of the form $S=\{\boldsymbol{h}_1,\dots,\boldsymbol{h}_n\}$ and a separate test set of the form $S_{test}=\{\boldsymbol{g}_1,\dots,\boldsymbol{g}_{n_{test}}\}$. For $p\in\{1,2\}$, let
\begin{align*}
L_{{\cal D}^*}^{L_p}=&\;\frac{1}{n}\sum_{i=1}^n\|\boldsymbol{h}_i\|_p^p+2pM^{p-1}\frac{C\|\boldsymbol{a}\|_2}{\sqrt{n}},\\ 
L_{test}^{L_p}=&\;\frac{1}{n}\sum_{i=1}^{n_{test}}\|\boldsymbol{g}_i\|_p^p.
\end{align*}
Let $t_{test}^{L_p}=\text{argmin}_t\  L_{test}^{L_p}$ and 
$t^*_{L_p}=\text{argmin}_t\  L_{{{\cal D}^*}}^{L_p}$.
(We have $L_{{\cal D}}^*\equiv L_{{\cal D}^*}^{L_1}$, 
$L_{test}\equiv L_{test}^{L_1}$,
$t^*\equiv t^*_{L_1}$, $t_{test}\equiv t_{test}^{L_1}$.) 
The value $t_{test}^{L_p}$ represents a good estimate of the 
{\em optimal} early stopping time with respect to $L_p$-norm.
On the other hand, it follows from Equation~\eqref{eq:LD0}  (and its counter part for $L_2$-norm) that, with probability at least
$1-\delta$:
\begin{equation}\label{eq:30}
    L_{test}^{L_p} \leq  L_{{\cal D}^*}^{L_p}+\epsilon_p.
\end{equation}
This is verified on the subsequent examples.
We also check that $t^*\equiv t^*_{L_1}$ is closer
than $t^*_{L_2}$
to~$t_{test}^{L_p}$ for both $p=1$ and $p=2$.
This confirms that our RC-based method works better with  $L_1$-norm 
  than with $L_2$-norm.
We focus on the case $m\leq n$ (underparameterization).
For $m>n$, we have  $L_{{\cal D}}^*(t)>0$ almost immediately for $t$ equal or close to $0$, and our method fails ($t^*\approx 0$).
In the examples, we check that our stopping rule is  beneficial (i.e., $L_{{\cal D}}^*(t^+)< L_{{\cal D}}^*(\infty)$).  
%
As shown in Example~\ref{ex:4}, the method applies equally to various kinds of data distribution (Gaussian, uniform, Pareto).
We observe  that it works better
(i.e., $t^*$ closer to $t_{test}$) when the ratio $n/m$ is larger. The numerical results are obtained using GD with a step size of $10^{-6}$.
For the sake of clarity, the values of times $t_1^+, t_\alpha^+, t^+, \dots$ are expressed as the number of GD steps. 

%
%
\begin{example}\label{ex:0} (Gaussian distribution) We consider the example of \cite{LiaoC18} in a binary classification setup. The input vector $\boldsymbol{h}_1,...,\boldsymbol{h}_n \in \mathbb{R}^m$ with $m=256$ are sampled from two distribution classes ${\cal C}_1$ and ${\cal C}_2$. A vector $\boldsymbol{h}_i$ belongs to the class ${\cal C}_j$ satisfies:
\begin{equation}
    \boldsymbol{h}_i = (-1)^j\boldsymbol{\mu} + \boldsymbol{z}_i
\end{equation}
for $j=\{1,2\}$, $\boldsymbol{\mu} = \left[2;\boldsymbol{0}_{m-1}\right]$, and the noise vector $\boldsymbol{z}_i\sim{\cal{N}}(\boldsymbol{0}_m,\boldsymbol{I}_m)$. To distinguish the two distribution, an input vector $\boldsymbol{h}_i$ \ is labeled by $y_i=1$ if $\boldsymbol{h}_i$ is in classes ${\cal C}_1$ and by $y_i = -1$ if $\boldsymbol{h}_i$ is in classes ${\cal C}_2$.
We first consider a training set of $n=512$ samples, with $216$ samples from ${\cal C}_1$ and the remaining $216$ from ${\cal C}_2$. 
We calculate the eigenvalues of the corresponding transition matrix~$\boldsymbol{H}$.
We have $\mathcal{V}_1:\{\lambda_1 = \lambda_\alpha=2865\}$ and $\mathcal{V}_2:\{\lambda_2= 1475, \dots,\lambda_n\}$.  
    We see that $dL_{{\cal D}}^*(t)/dt$ (computed using~\eqref{eq: 67})
    becomes 0 at $t=t^*=27<\tau_1=100$. This corresponds to case 1 of Proposition~\ref{prop:1}\ \  ($t^*\in {\cal T}_1=[0,\tau_1)$). The estimate~$t^+= 27$ (computed using~\eqref{eq:t+}) satisfies $t^+= t^*=27$ in accordance with Equation~\eqref{eq:case1}. See Figure~\ref{fig: RMT 1}. We see in
    Figure~\ref{fig: RMT 512 L test} that $t^* \equiv t_{L_1}^*$ is closer than $t_{L_2}^*$ to $t^{L_p}_{test}$ for both $p = 1$ and $p = 2$.

    Similarly, for $n = 16384$ and $m = 256$,
    the derivative $dL_{{\cal D}}^*(t)/dt$ 
    becomes~0 at $t=t^*=64<\tau_1=100$. This corresponds again to case 1 of Proposition~\ref{prop:1}. 
     The estimate~$t^+$ 
    satisfies
    $t^+= t^*=64<\tau_1$, in accordance with Equation~\eqref{eq:case1}.    See Figure~\ref{fig: RMT 16384}.
    Likewise in Figure~\ref{fig: RMT 16384 L test}, we see that 
    $t^*\equiv t^*_{L_1}$ is  closer than $t^*_{L_2}$ to $t_{test}^{L_p}$ for both $p=1$ and $p=2$.

\begin{figure}
    \centering
    \includegraphics[width=0.5\linewidth]{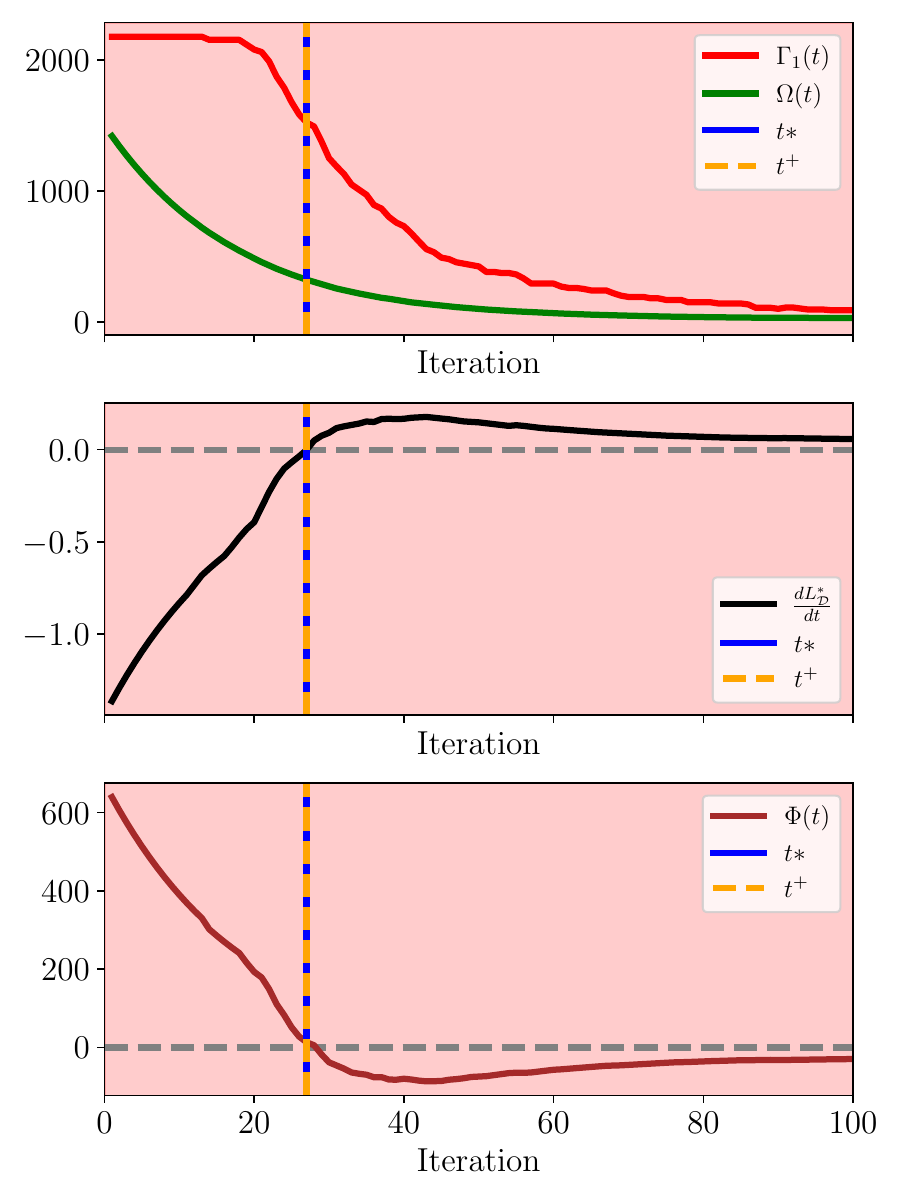}
    \caption{Top: the sign-invariant subinterval ${\cal T}_1 = \left[0,100\right)$ with curves $\Gamma_1,\Omega$.
    Middle: curve $dL_{{\cal D}}^*/dt$. 
    Bottom: curve $\Phi(t)$. 
    We have $t^+= t^*=27<\tau_1=100$ (case~1 of Proposition~\ref{prop:1}).}
    \label{fig: RMT 1}
\end{figure}

\begin{figure}
    \centering
    \includegraphics[width=0.5\linewidth]{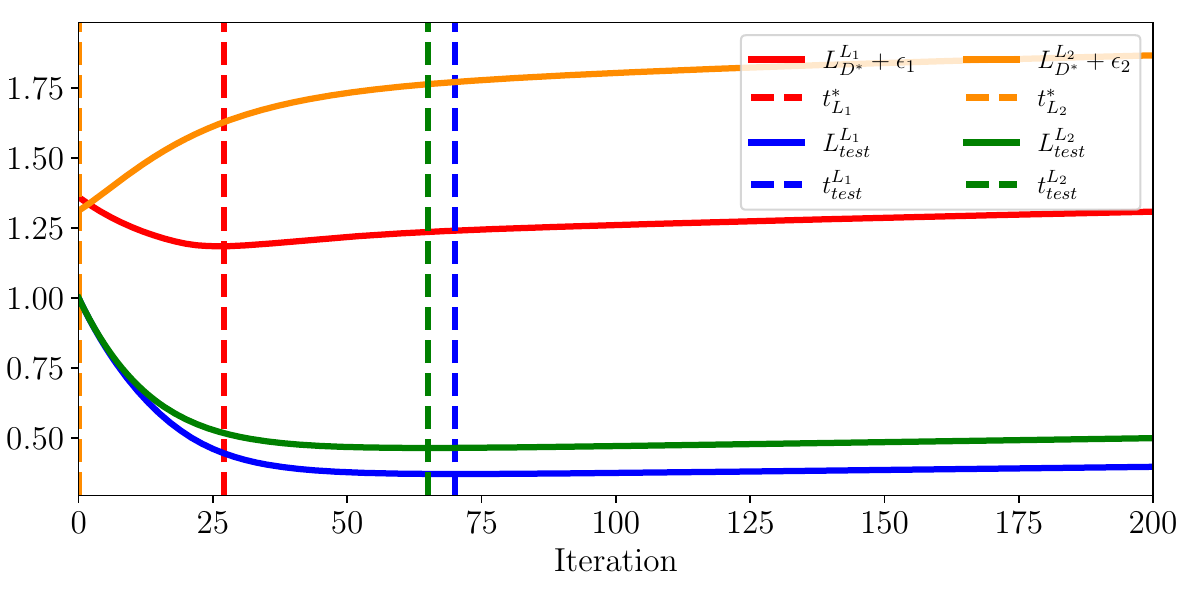}
    \caption{Curves $L_{{\cal D}^*}^{L_p}$ and $L_{test}^{L_p}$  for Gaussian distribution ($m=256$, $n= 512$). We check that $L_{test}^{L_p}$ is below $L_{{\cal D}^*}^{L_p}$ ($p=1,2$), and $t^*_{L_1}$ closer than $t^*_{L_2}$ to $t_{test}^{L_p}$ for $p=1$ and~$2$.}
    \label{fig: RMT 512 L test}
\end{figure}

\begin{figure}
    \centering
    \includegraphics[width=0.5\linewidth]{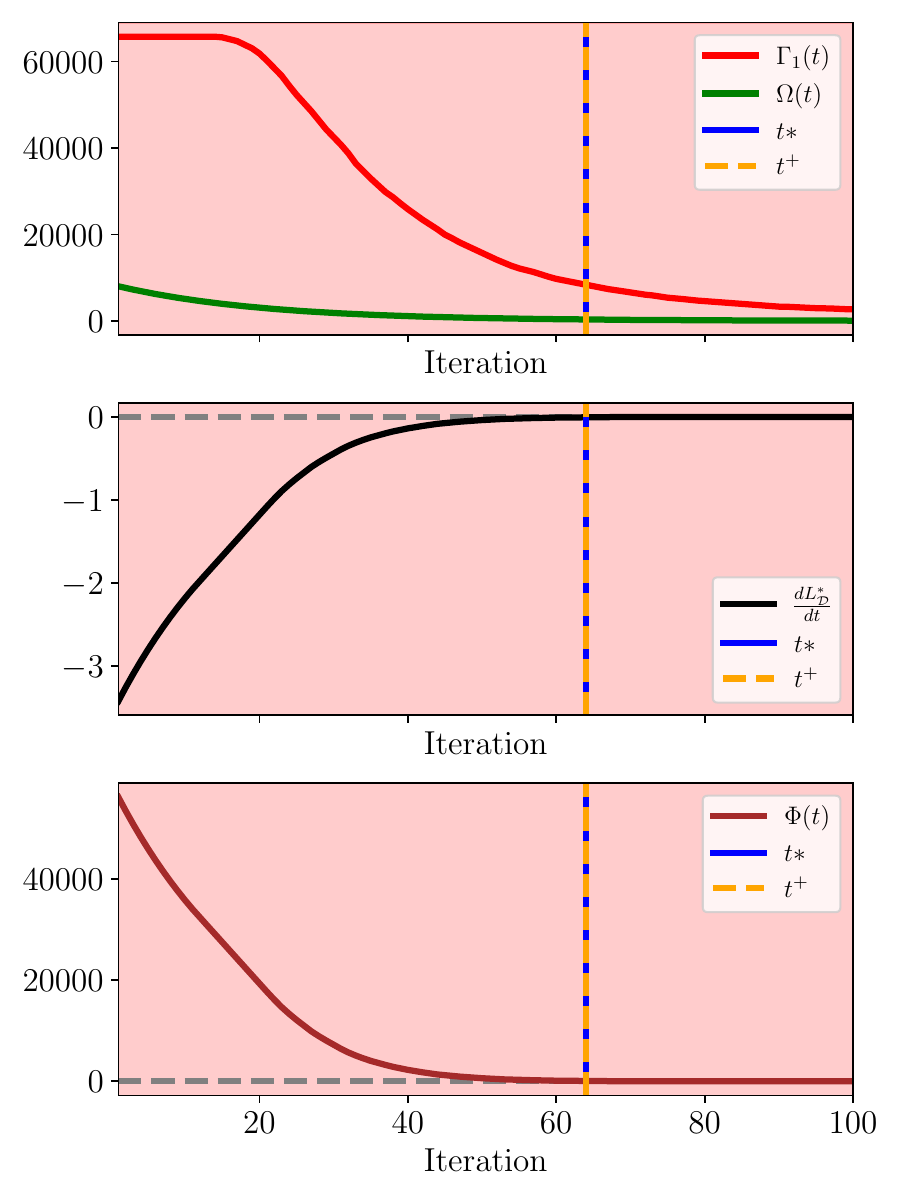}
    \caption{Top: the sign-invariant subinterval ${\cal T}_1 = \left[0,100\right)$ with curves $\Gamma_1,\Omega$.
    Middle: curve $dL_{{\cal D}}^*/dt$. 
    Bottom: curve $\Phi(t)$. 
    We have $t^+= t^*=64<\tau_1$ (case~1). 
    }
    \label{fig: RMT 16384}
\end{figure}

\begin{figure}
    \centering
    \includegraphics[width=0.5\linewidth]{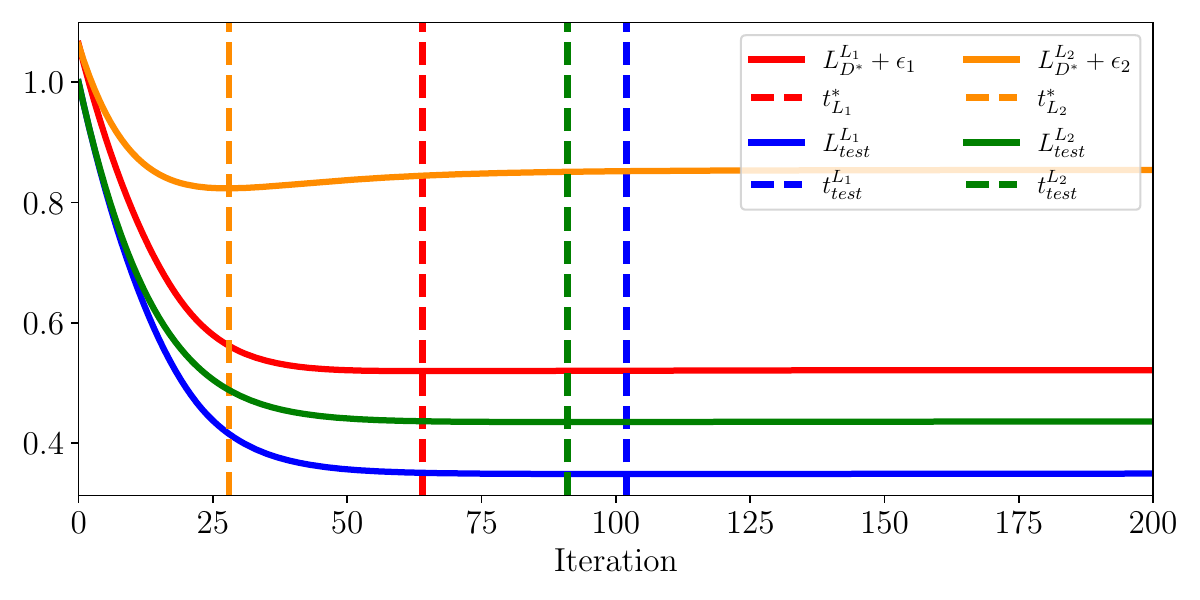}
    \caption{Curves $L_{{\cal D}^*}^{L_p}$ and $L_{test}^{L_p}$ for Gaussian distribution ($m=256$, $n= 16384$). We check that $L_{test}^{L_p}$ is below $L_{{\cal D}^*}^{L_p}$ ($p=1,2$), and $t^*_{L_1}$ closer than $t^*_{L_2}$ to $t_{test}^{L_p}$ for $p=1$ and~$2$.}
    \label{fig: RMT 16384 L test}
\end{figure}
\end{example}

\begin{example}\label{ex:1} (MNIST classification 3-5)
%
We consider a problem of binary classification between
classes 3 and 5 of the MNIST dataset
\cite{lecun1998mnist}.
We consider an NN 
with ReLU activation function,
4 hidden layers of width $m=10$ and output layer
$\boldsymbol{a}\in\mathbb{R}^m$.
Using Pytorch \cite{PaszkeGMLBCKLGA19}, we select a training set~$S$ of $n=10^4$ samples of the form $(\boldsymbol{x}_i,y_i)$ (with $\boldsymbol{x}_i\in\mathbb{R}^{784}$) and a test set $S_{test}$ of $1900$ samples. 
%
%
Following the {\em linear probing} method 
(see \cite{alain2016}),
we reduce the NN to a linear model as follows.
We first freeze each weight $a_i$ ($i\in[m]$) of $\boldsymbol{a}$ arbitrarily to either~$1$ or $-1$, and
pre-train the resulting model $NN$ using GD
(with a step size $\eta=10^{-6}$). The weights of the hidden layers
are then frozen themselves, which yields
$n$ fixed vectors of the form $\boldsymbol{h}_1=NN(\boldsymbol{x}_1),\dots, \boldsymbol{h}_n=NN(\boldsymbol{x}_n)\in\mathbb{R}^m$.
We then regard $\boldsymbol{h}_1,\dots, \boldsymbol{h}_n$ 
as input vectors, and consider
the linear model consisting only of the output layer $\boldsymbol{a}$ reinitialized
to~0.
%
%
We calculate the eigenvalues of the corresponding
transition matrix $\boldsymbol{H}$. 
We have ${\cal{V}}_1:\{\lambda_1 = 12919,\lambda_2 \equiv \lambda_\alpha=11384$\}, 
and ${\cal{V}}_2:\{\lambda_3 = 48, \dots, \lambda_n\}$.
%
\begin{figure}[htp]
    \centering
    \includegraphics[scale = 0.4]{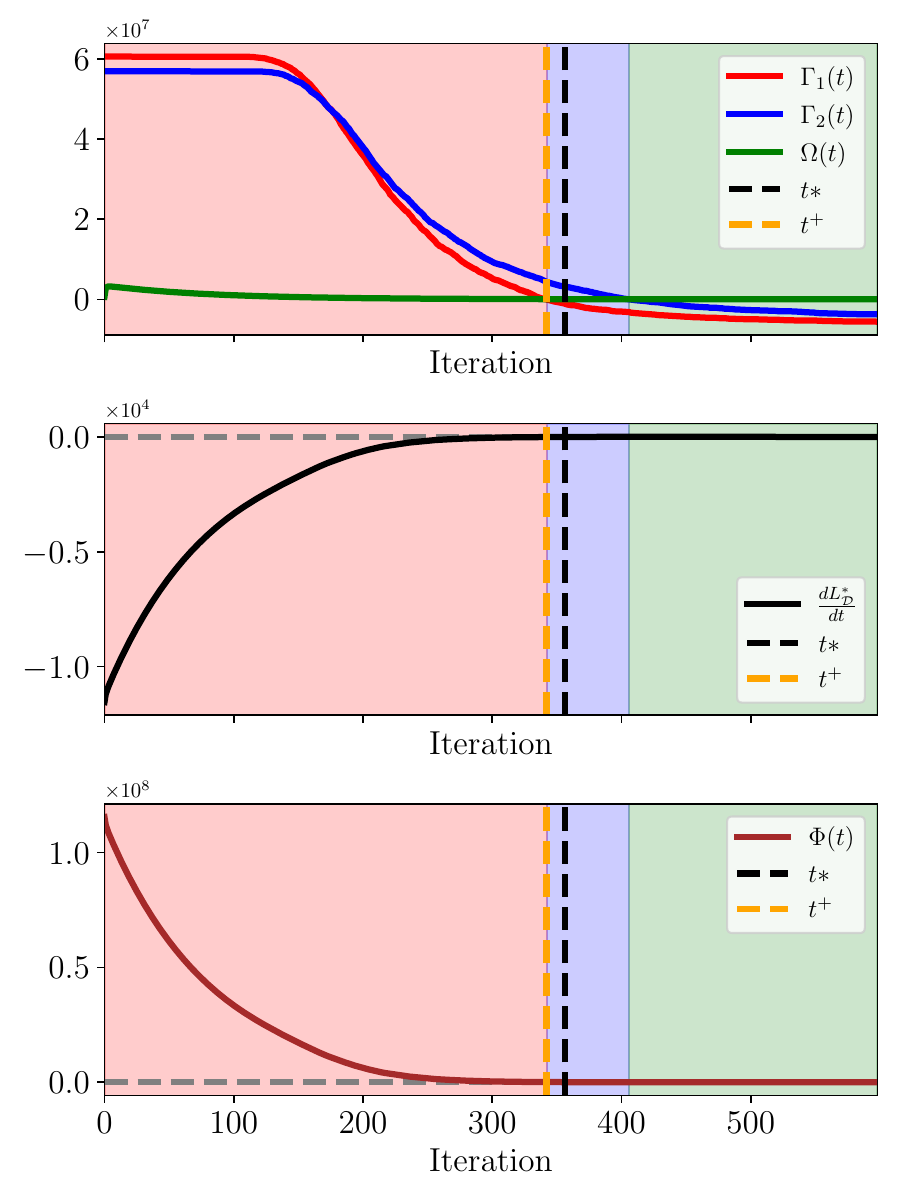}
  %
    \caption{Top: sign-invariant subintervals ${\cal T}_1 = \left[0,343\right)$,
 ${\cal T}_2 = \left[343,407\right)$, 
 ${\cal T}_3 = \left[407,600\right)$ with curves $\Gamma_1,\Gamma_2,\Omega$. Middle: curve $dL_{{\cal D}}^*/dt$. 
    Bottom: curve~$\Phi(t)$. 
    We have $t^+= 342\approx \tau_1=343\leq t^*=357$ (case 2 of Proposition~\ref{prop:1}).}
    \label{fig: check dLDdt}
\end{figure}
%
We consider the time interval $T = \left[0,600\right)$, which divides it into 
three sign-invariant subintervals:
${\cal T}_1 = \left[0,343\right)$ (in red, the top plot of Figure~\ref{fig: check dLDdt}),
 ${\cal T}_2 = \left[343,407\right)$ (in  blue), 
 and 
 ${\cal T}_3 = \left[407,600\right)$ (in green). 
 We find $C=1.572$, $M=2$ and $\epsilon=0.0814$.  Using~\eqref{eq: 67}, we
  find  $dL_{{\cal D}}^*(t)/dt\geq 0$  for the first time
    at $t^*=357$ (see
    the middle plot of Figure~\ref{fig: check dLDdt}). 
 On the other hand,  we have:
    $I_-=\emptyset$ and 
    $\Gamma(0)>\Omega(0)>0$ on~${\cal T}_1$, so~\eqref{eq:spec} is  satisfied.
 %
    We find  $t^+= 342$ using Equation~\eqref{eq:t+}
    (see the bottom plot of Figure~\ref{fig: check dLDdt}),
    and verify $t^+_1=342\leq t^+\leq t^+_\alpha=342$
    in accordance with Equation~\eqref{eq:encadrement}. 
    Since $t^*=357>\tau_1=343$, we have $t^*\not\in {\cal T}_1$
    (case~2 of Proposition~\ref{prop:1}).  So
$t^+= 342\approx\tau_1=343\leq t^*=357$,
    in accordance with~\eqref{eq:case2}.

    
%

\begin{figure}[htp]
    \centering
    \includegraphics[scale = 0.35]{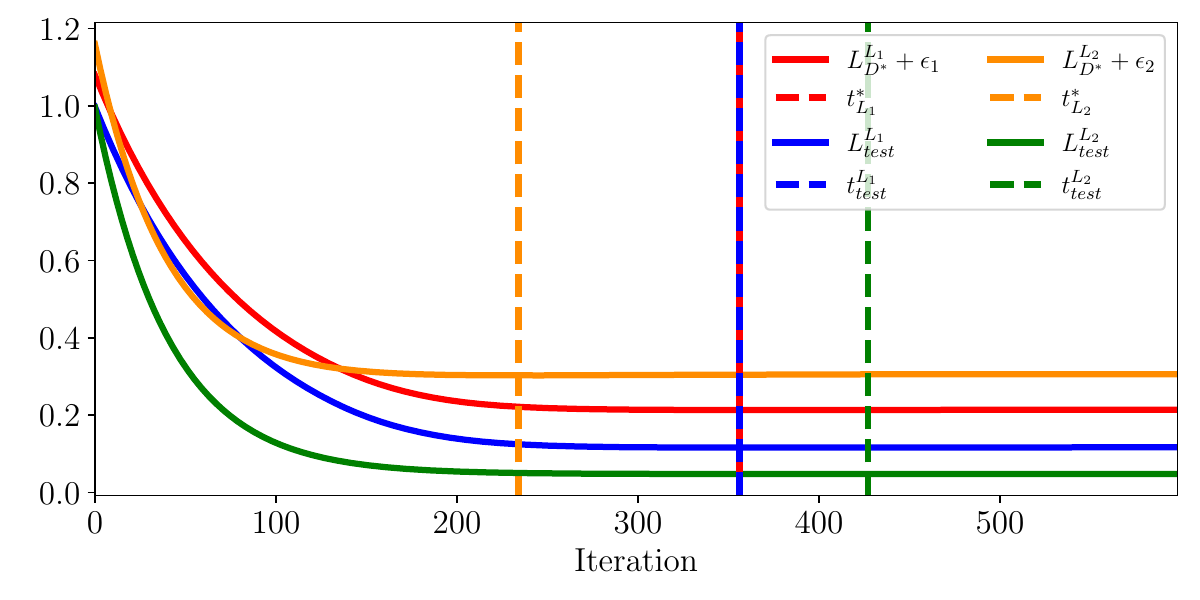} 
    \caption{Curves $L_{{\cal D}^*}^{L_p}+\epsilon_p$ and $L_{test}^{L_p}$  
    for classification 3-5. We check that $L_{test}^{L_p}$ is below $L_{{\cal D}^*}^{L_p}+\epsilon_p$ ($p=1,2$), and $t^*_{L_1}$ closer than $t^*_{L_2}$ to $t_{test}^{L_p}$ for $p=1$ and $p=2$.
    }
    \label{fig: L test 3 5}
\end{figure}
%
%
Figure~\ref{fig: L test 3 5} shows that, for $\delta=0.05$,  
curve $L_{test}$ is always below $L_{{\cal D}}^* +\epsilon$, in accordance with equation~\eqref{eq:30}.  We see that $t_{test}=356$ almost coincides with
$t^*=357$. 
So curves
$L_{test}$ and~$L_{{\cal D}}^*$ reach their  minima
at almost the same  time. 
The values of $L_{test}(t)$ at $t=t^+, t^*, t_{test}$ are identical (equal to~$0.1166$), and almost equal to
its value  at $t=t_{\text{approx}}^+$ (equal to~$0.1167$).
See Table~1.
%
This shows an excellent agreement between analytic and numerical estimates of the optimal stopping time.
%
Besides, we have:
$L_{\cal{D}}^*(t^+)+\epsilon = 0.2138 < L_{\cal{D}}^*(\infty)+\epsilon = 0.2274$, which shows that early stopping is beneficial. 


\begin{table}[]
    \centering
    \begin{tabular}{|c|c|c|c|c|c|c|c|}
    \hline
         & $t_1^+$& $t_\alpha^+$& $t^+$& $t^+_{\text{approx}}$& $t^*$& $t_{test}$& $\infty$ \\
    \hline
        $t$ & 342 & 342 & 342 & 400& 357& 356& $\infty$\\
        $L_{test}$ & 0.1166 & 0.1166 & 0.1166 & 0.1167& 0.1166& 0.1166 & 0.1172 \\
         $L_{\cal{D}}^*+\epsilon$ & 0.2138 & 0.2138 & 0.2138 & 0.2139& 0.2138&0.2138& 0.2274\\
    \hline
    \end{tabular}
    \caption{The values of $t^+_1$, $t^+_\alpha$ ($\alpha=2$), $t^+$, $\dots$, $t_{test}$ with corresponding $L_{test}(t)$ and $L_{\cal{D}}^*(t)+\epsilon$ \ \ 
    (Example~\ref{ex:1}).}
    \label{tab:placeholder}
\end{table}

\end{example} 

\begin{example}\label{ex:2} (MNIST classification 0-1)
%
We use the same framework as in Example~\ref{ex:1} with a 4-hidden-layer NN of width $m=10$, 
and a training set $S$ of $n=10^{4}$ examples, corresponding  to classes~0 and~1 (instead of~3 and~5).
We calculate the eigenvalues of the corresponding
transition matrix $\boldsymbol{H}$. 
We have ${\cal{V}}_1:\{\lambda_1 = 11432,\lambda_2 \equiv \lambda_\alpha=10639$\}
and ${\cal{V}}_2$: $\{\lambda_3 = 0.005,\dots,\lambda_n\}$.
%
\begin{figure}[htp]
    \centering
    \includegraphics[scale = 0.4]{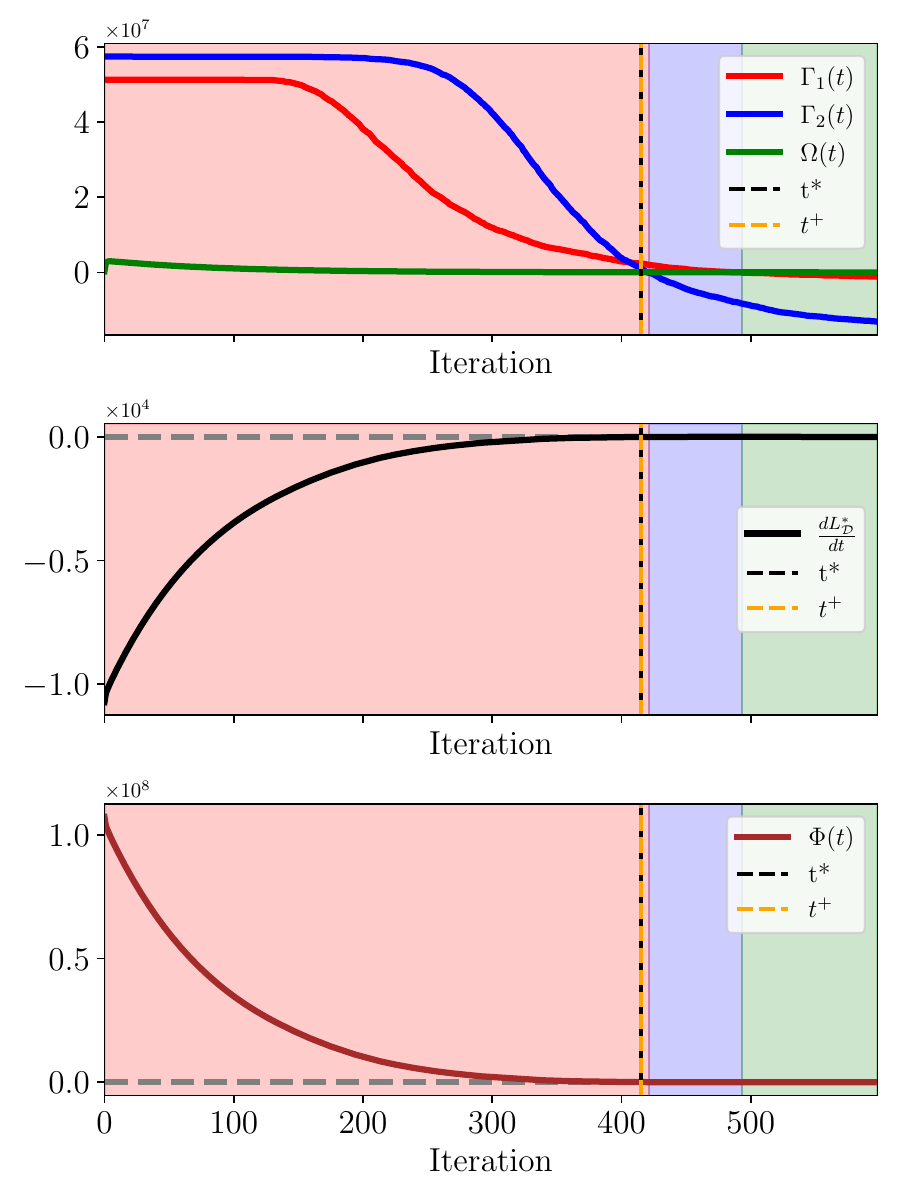}

    \caption{Top: sign-invariant subintervals ${\cal T}_1 = \left[0,421\right)$,
 ${\cal T}_2 = \left[421,493\right)$,
 ${\cal T}_3 = \left[493,600\right)$ with curves $\Gamma_1,\Gamma_2,\Omega$. Middle: curve $dL_{{\cal D}}^*/dt$. 
    Bottom: curve~$\Phi(t)$. 
    We have
    $t^+=t^*=415<\tau_1=421$ (case 1).}
    \label{fig: check dLDdt 0 1}
\end{figure}
%
%
The time interval $T = \left[0,600\right)$ divides into sign-invariant subintervals
${\cal T}_1 = \left[0,421\right)$ (in red, the top plot of Figure~\ref{fig: check dLDdt 0 1}),
${\cal T}_2 = \left[421,493\right)$ (in blue), 
and 
${\cal T}_3 = \left[493,600\right)$ (in green).
We find $C=1.495$, $M=2$ and $\epsilon=0.0814$. 
 Using~\eqref{eq: 67}, 
    we find 
    $dL_{{\cal D}}^*/dt= 0$ 
    at $t^*=415$ (see the middle plot
    Figure~\ref{fig: check dLDdt 0 1}).
    Here again, 
     \eqref{eq:spec} holds  on~${\cal T}_1$.
We find  $t^+= 415$ (using Equation~\eqref{eq:t+}), and verify
$t_1^+=414\leq t^+\leq t_\alpha^+\equiv t_2^+=417$,
in accordance with Equation~\eqref{eq:encadrement}.
    See the bottom plot of Figure~\ref{fig: check dLDdt 0 1}.
Since $t^*=415<\tau_1=421$, we have $t^*\in{\cal T}_1$
    (case~1 of Prop.~\ref{prop:1}).
     We have: $t^+=t^*=415<\tau_1=421$, 
    in accordance with~\eqref{eq:case1}.
    


\begin{figure}[htp]
    \centering
    \includegraphics[scale = 0.35]{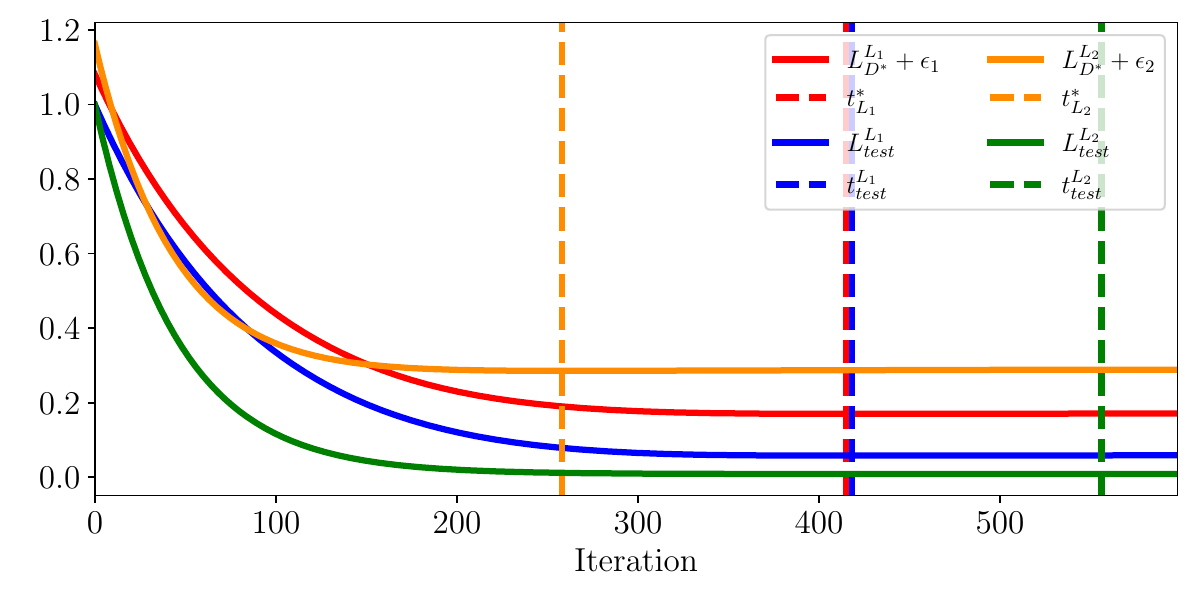}  
    \caption{Curves $L_{{\cal D}^*}^{L_p}+\epsilon_p$ and $L_{test}^{L_p}$  for classification 0-1. We check that $L_{test}^{L_p}$ is below $L_{{\cal D}^*}^{L_p}+\epsilon_p$ ($p=1,2$), and $t^*_{L_1}$ closer than $t^*_{L_2}$ to $t_{test}^{L_p}$ for $p=1$ and $p=2$.
    }
    \label{fig: L test 0 1}
\end{figure}

Figure~\ref{fig: L test 0 1} shows that curve
$L_{test}$ is  again below $L_{{\cal D}}^* + \epsilon$ 
in accordance with~\eqref{eq:30}.
We have
$t_{test}=418$, which is close to 
$t^*=415$. So
$L_{test}$ and $L_{{\cal D}}^*$ reach their  minima
at around the same time ($t^+=t^*=415\approx t_{test}^{L_1}=418$).
The values of $L_{test}(t)$ at $t=t^+, t^*, t_{test}$ are identical ($=0.0582$), and almost equal to
the value at $t=t_{\text{approx}}^+$ ($=0.0586$). See Table~2.
Here again, 
there is  an excellent agreement
between analytic and numerical estimates of the optimal
stopping time.
%
\begin{table}[]
    \centering
    \begin{tabular}{|c|c|c|c|c|c|c|c|}
    \hline
         & $t_1^+$ & $t_\alpha^+$ & $t^+$& $t^+_{\text{approx}}$& $t^*$& $t_{test}$& $\infty$ \\
    \hline
        $t$ & 414 & 417 &  415 & 508& 415& 418& $\infty$\\
        $L_{test}$ & 0.0582 & 0.0582 & 0.0582 & 0.0586& 0.0582& 0.0582& 0.0590 \\
         $L_{\cal{D}}^*+\epsilon$ &0.1703 & 0.1703 &  0.1703 &  0.1707& 0.1703& 0.1703& 0.1713\\
    \hline
    \end{tabular}
    \caption{The values of $t_1^+$, $t_\alpha^+$ ($\alpha=2$), $t^+$, $\dots$, $t_{test}$ with corresponding $L_{test}(t)$ and $L_{\cal{D}}^*(t)+\epsilon$ \ \  
    (Example~\ref{ex:2}).}
    \label{tab: 0 1}
\end{table}
Finally, we have:
$L_{\cal{D}}^*(t^+)+\epsilon = 0.1703 < L_{\cal{D}}^*(\infty)+\epsilon = 0.1713$, which proves that early stopping is again beneficial. 
\end{example}
\begin{example}\label{ex:4} (Various input data distributions)
    For each kind of distribution, we give a figure displaying the
    curves $L_{{\cal D}^*}^{L_p}+\epsilon_p$ and $L_{test}^{L_p}$
    with two  couples $(m,n)$ of different ratio~$m/n$: Figure~\ref{fig:9} for Gaussian distribution, Figure~\ref{fig:10} for uniform distribution, Figure~\ref{fig:13} for Pareto distribution.
    For all these examples, we have $\alpha=1$.
    The values of $\lambda_1$ and $\lambda_{2}$  are  given in caption.  
    We see that the method works better (i.e., $t^*$ closer to $t_{test}$) when the ratio $n/m$ is larger. For example,
    in Figure~\ref{fig:13}, the method fails for  $m=256, n=512$
    (top plot: $t^+=t^*=0$) while it succeeds for  $m=256, n=16384$ (bottom  plot: $t^+=t^*=45\approx  t_{test}=46$).

\begin{figure}[ht!]
    \centering
    \includegraphics[width=0.5\linewidth]{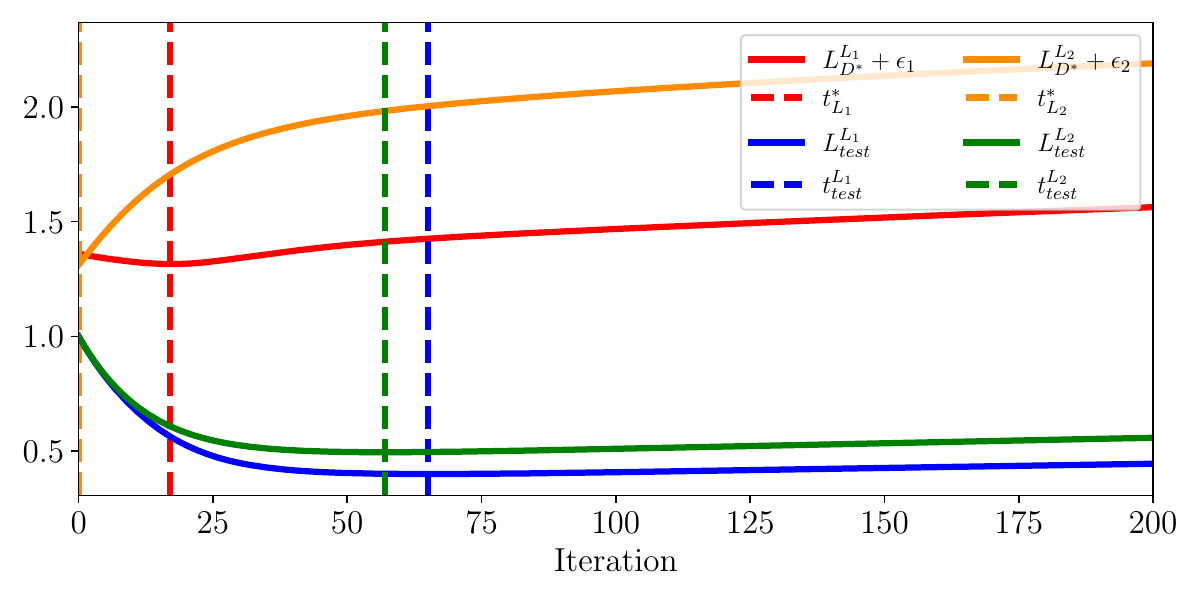}
    \includegraphics[width=0.5\linewidth]{Figure/Compare_l1_l2_n=16384.pdf}
    \caption{Gaussian distribution. Top: $m=512$, $n= 512$
    ($\lambda_1 \equiv \lambda_\alpha=2291$, $\lambda_2= 486$).
    Bottom: $m=256$, $n= 16384$ ($\lambda_1 \equiv \lambda_\alpha=8.23\cdot 10^4$, $\lambda_2= 2.05\cdot 10^4$).
    }
    \label{fig:9}
\end{figure}


\begin{figure}[ht!]
    \centering
    \includegraphics[width=0.5\linewidth]{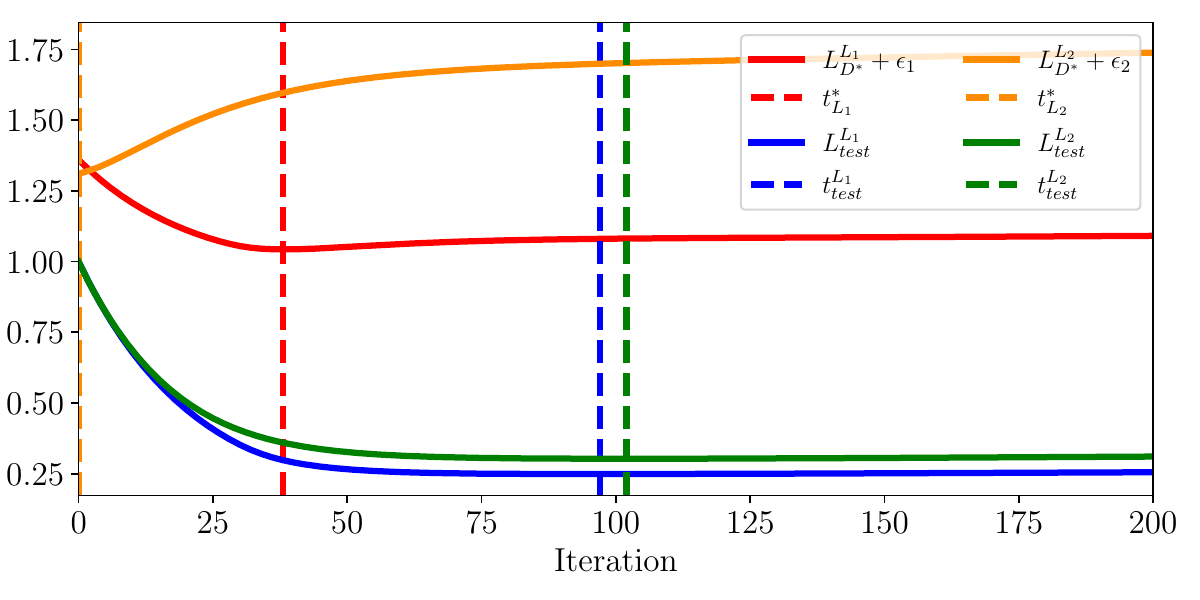}
    \includegraphics[width=0.5\linewidth]{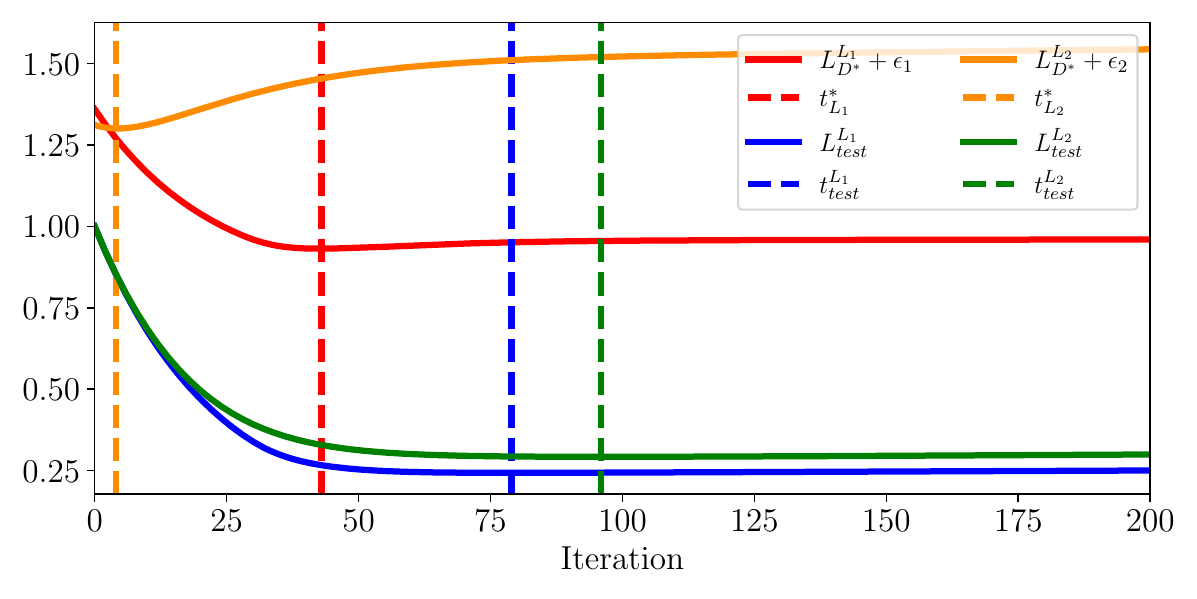}
    \caption{
    Uniform distribution.
    Top: $m = 512$, $n= 512$ ($\lambda_1 \equiv  \lambda_\alpha=2413$, $\lambda_2 = 679$).
    Bottom: $m=256$, $n= 512$ ($\lambda_1\equiv\lambda_\alpha= 2291$, $\lambda_2 = 486$).}
    \label{fig:10}
\end{figure}


\begin{figure}[ht!]
    \centering
    \includegraphics[width=0.5\linewidth]{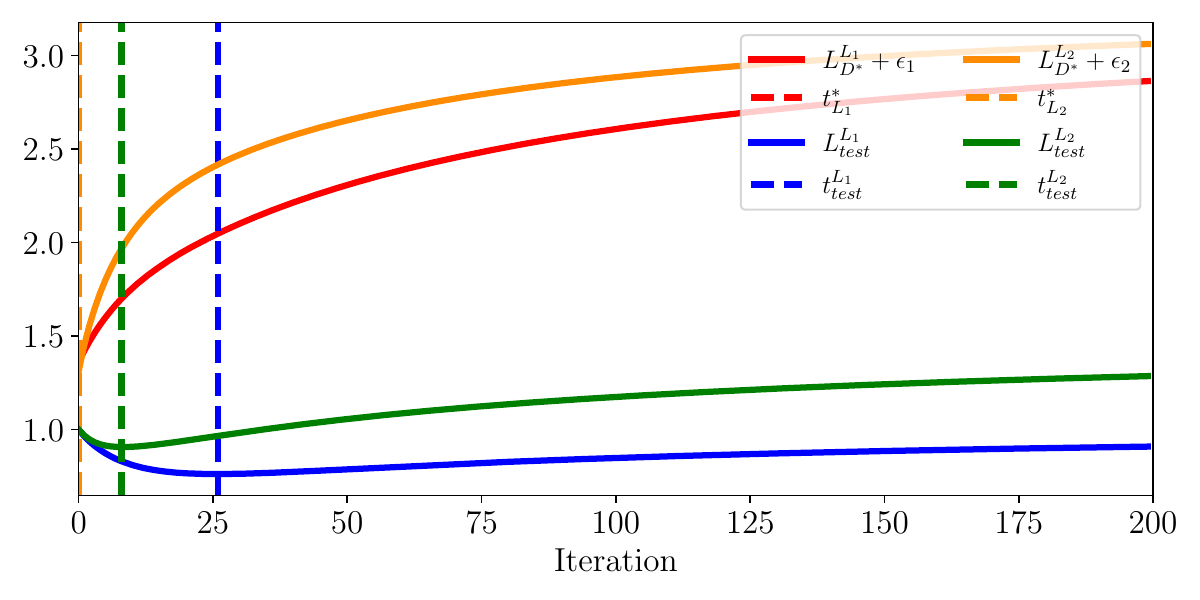}
    \includegraphics[width=0.5\linewidth]{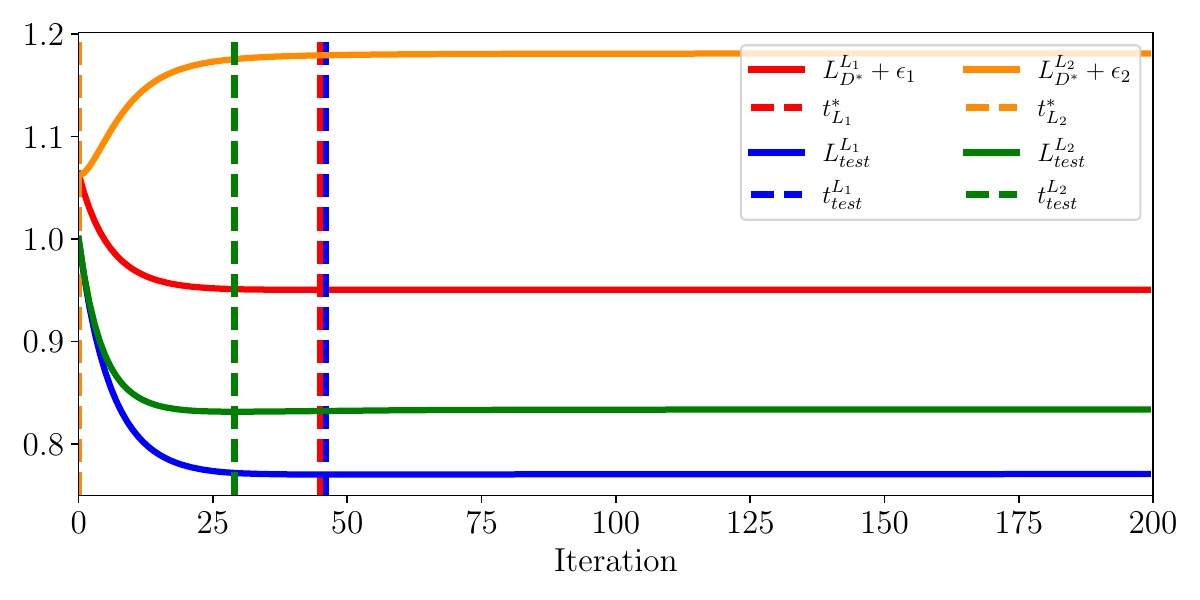}
    \caption{
    Pareto distribution. 
    Top: $m=256$, $n= 512$.
    ($\lambda_1 \equiv\lambda_\alpha= 5\cdot 10^5$, $\lambda_2 = 1\cdot 10^5$).
    Bottom: $m=256$, $n= 16384$.  
    ($\lambda_1 \equiv\lambda_\alpha=17\cdot 10^6$, $\lambda_2 = 3\cdot 10^6$).
    }
    \label{fig:13}
\end{figure}
\end{example}

\section{Final Remarks}









We have proposed a data-dependent definition of time $t^+$ as a stopping time for the gradient flow process.
This time $t^+$ is a low estimate of the time $t^*$ at which
an upper bound on the
population loss $L_{{\cal D}}$ reaches its minimum.
The advantage is that $t^+$
(and even more so its approximation $t^+_{\text{approx}}$) is easily calculated
without needing to apply gradient~flow. 
%
%
%
Our method is well suited to the underparameterized context ($m\leq n$), with results improving as the ratio $n/m$ increases.
It  applies equally to different types of  data distributions without requiring any specific knowledge on them.

We have focused on linear models. Although the method, in combination with linear probing, has been applied successfully to nonlinear examples, it would be interesting to apply it directly to nonlinear models. 
On the other hand, we have considered only scalar outputs ($y\in\mathbb{R}$). Another extension would be  to consider vectorial outputs ($y\in\mathbb{R}^q$ with $q\geq 2$). Finally, it would be interesting to treat input data affected by random noise.
All these extensions will be the subject of future work.
\newpage

\appendix
\section{Proof of Proposition \ref{prop:Rademacher L1}}
\begin{proof}
The proof is a simple adaptation of the proof of Proposition~3 of~\cite{XavierCF25} in the case of linear models. 
%
%
By Theorem~(11.3) of \cite{mohri2018foundations},
we have with probability~$1-\delta$:
$$L_{{\cal D}}-L_S\leq 2{\cal R}_n+\epsilon$$
where ${\cal R}_n$ is the Rademacher complexity of the linear class.
By Theorem 5.5 of \cite{ma2022}, we have
${\cal R}_n\leq \|\boldsymbol{a}\|_2 C/\sqrt{n}$.
It then follows:
$$L_{{\cal D}}\leq L_S+\frac{2C\|\boldsymbol{a}\|_2}{\sqrt{n}}+\epsilon=L_{{\cal D}}^*+\epsilon,$$
i.e. 
\eqref{eq:LD0}.
\end{proof}

\section{Proof of Proposition \ref{prop: 3.1}}
\begin{proof}
    Using \eqref{eq:LS} and \eqref{eq: v dynamic ln}, the derivative of $L_S$ is:
\begin{equation}\label{eq: 71}
\begin{aligned}
     \frac{dL_S(t)}{dt} =&\; \frac{1}{n}\frac{d\|\boldsymbol{v}(t)\|_1}{dt} =\frac{1}{n}\frac{\partial\|\boldsymbol{v}\|_1}{\partial \boldsymbol{v}}\frac{d\boldsymbol{v}(t)}{dt}
     =
     -\frac{1}{n}\left(\text{sgn}(\boldsymbol{v}(t))\right)^\top\boldsymbol{H}\boldsymbol{v}(t).
\end{aligned}
\end{equation}

Besides, using \eqref{eq:LG}, the derivative of $L_G^*$ is:
\begin{equation}
\begin{aligned}
\label{eq: 72}
     \frac{dL_G^*(t)}{dt} =&\; \frac{2C\boldsymbol{a^\top}}{\sqrt{n}\|\boldsymbol{a}\|_2}\frac{d\boldsymbol{a}}{dt}.
\end{aligned}
\end{equation}
From \eqref{eq:vv}, we have:
$\boldsymbol{K}\boldsymbol{a}(t) = \boldsymbol{v}(t) + \boldsymbol{y}$.
Since $\boldsymbol{a}(t)$ is initialized to 0,  
 know that
the vector $\boldsymbol{a}(t)$ updated by GF, satisfies
(see, e.g., \cite{bjorck1973}):
\begin{equation}\label{eq: a}
\boldsymbol{a}(t) = \boldsymbol{K}^\dagger\left(\boldsymbol{v}(t) + \boldsymbol{y}\right).
\end{equation}
So, using  \eqref{eq:da}, \eqref{eq:13bis} and \eqref{eq: a}, Equation~\eqref{eq: 72} becomes :
\begin{equation}
\begin{aligned}
\label{eq: 72bis}
     \frac{dL_G^*(t)}{dt} =&\; -\frac{2C}{\sqrt{n}}\frac{(\boldsymbol{K}^\dagger(\boldsymbol{v}(t)+\boldsymbol{y}))^\top}{\|\boldsymbol{K}^\dagger(\boldsymbol{v}(t)+\boldsymbol{y})\|_2} \boldsymbol{K}^\top v(t) =\Psi(t).
\end{aligned}
\end{equation}

It then follows from \eqref{eq:L_D*}, \eqref{eq: 71} and \eqref{eq: 72bis}:
\begin{equation*}
\begin{aligned}
\frac{dL_{{\cal D}}^*(t)}{dt} =& -\frac{1}{n}\left(\text{sgn}(\boldsymbol{v}(t))\right)^\top\boldsymbol{H}\boldsymbol{v}(t) +\Psi(t),
\end{aligned}
\end{equation*}
i.e. \eqref{eq: 67}.
\end{proof}
\section{Proof of Proposition \ref{prop:1}}
\begin{proof}
We have 
\begin{align*}
    \frac{dL_S(t)}{dt} =&\; -\frac{1}{n}\left(\text{sgn}(\boldsymbol{v}(t))\right)^\top\boldsymbol{H} \boldsymbol{v}(t)\\
    =&\;-\frac{1}{n}\left(\text{sgn}(\boldsymbol{v}(t))\right)^\top\boldsymbol{H}\sum_{i=1}^\alpha \boldsymbol{w}_i e^{-\lambda_it}\\
    &\; -\frac{1}{n}\left(\text{sgn}(\boldsymbol{v}(t))\right)^\top\boldsymbol{H}\sum_{j=\alpha+1}^n \boldsymbol{w}_j e^{-\lambda_jt}\\
    =&\;-\frac{1}{n}\left(\text{sgn}(\boldsymbol{v}(t))\right)^\top\boldsymbol{H}\sum_{i=1}^\alpha \boldsymbol{w}_i e^{-\lambda_it}-\frac{1}{n}\Delta(t).
\end{align*}

Hence, using \eqref{eq: 67}, we have\\

$dL_{{\cal D}}^*(t)/dt=\frac{1}{n}dL_S(t)/dt+\Psi(t)<0$ iff:
\begin{align*}
    -&\sum_{i=1}^\alpha \left(\text{sgn}(\boldsymbol{v}(t))\right)^\top\boldsymbol{H}\boldsymbol{w}_i e^{-\lambda_it}-\Delta(t)+n \Psi(t)<0\\
\intertext{i.e.:}
    -&\sum_{i=1}^\alpha \left(\text{sgn}(\boldsymbol{v}(t))\right)^\top\boldsymbol{H}\boldsymbol{w}_i e^{-\lambda_it}+\Omega(t)<0\\
\intertext{i.e.:}
    &\sum_{i\in [\alpha]}\Gamma_i(t)e^{-\lambda_i t} - \Omega(t)>0\\
\intertext{i.e.:}
    &\;\Phi(t)> 0.
\end{align*}
Hence $dL_{{\cal D}}^*(t)/dt<0$ iff $\Phi(t)>0$,
i.e.: ~\eqref{eq:basic00}. 
From \eqref{eq:t+} and \eqref{eq:basic00}, it then follows: 
\begin{equation*} 
t^+=\sup_{s\in {\cal T}_1}\left\{s: \frac{dL_{{\cal D}}^*(s)}{dt}<0\,\ \forall t\in[0,s)\right\},
\end{equation*}
i.e. \eqref{eq:decrease0}.
On the other hand, $t^*$ is the least value $s\geq 0$ such that $dL_{{\cal D}}^*(s)/dt=0$
(see~\eqref{eq:t**}).  Hence
$t^+\leq t^*$,
i.e. \eqref{eq:21bis}.

There are now two cases:
Either~$t^*$~belongs  to ${\cal T}_1$, i.e. $t^*<\tau_1$ (case 1), or not, i.e. $t^*\geq \tau_1$ (case~2).

Case 1 ($t^*<\tau_1)$: Since $t^*$ is the first
time $s$ such that $dL_{{\cal D}}^*(s)/dt\geq 0$ (see \eqref{eq:t**}), we have
$dL_{{\cal D}}^*(s)/dt <0$ for all $s<t^*<\tau_1$. It  follows from~\eqref{eq:decrease0}: $t^*\leq t^+$. Since $t^+\leq t^*$
(see \eqref{eq:21bis}), we have: $t^+=t^*$, i.e.~\eqref{eq:case1}.

Case 2 ($t^*\geq\tau_1$): 
Since $t^*$ is the first
time $s$ such that $dL_{{\cal D}}^*(s)/dt\geq 0$ (see \eqref{eq:t**}) and $t^*\geq \tau_1$, we have
$dL_{{\cal D}}^*(s)/dt<0$ for all $s<\tau_1$.
It follows from~\eqref{eq:decrease0}: $\tau_1\leq t^+$.
Besides,   as a supremum of a set of elements~$<\tau_1$, $t^+$ satisfies: $t^+\leq \tau_1$. So $t^+=\tau_1$, i.e.
\eqref{eq:case2}.

Suppose furthermore \eqref{eq:spec}.
Then
$\Gamma(t)=\Sigma_{i\in[\alpha]}\Gamma_i(t)=\Sigma_{i\in{\cal I}^+}\Gamma_i(t)$, with $\Gamma_i(t)> 0$ for all $i\in[\alpha] $.
Then, using \eqref{eq:CS3}, we have 
$$\Gamma(t)e^{-\lambda_1 t}-\Omega(t)\leq\Phi(t)\leq \Gamma(t)e^{-\lambda_\alpha t}-\Omega(t).$$
Using~\eqref{eq:t+}:
\begin{equation*} 
t^+=\sup_{s\in {\cal T}_1}\left\{s: \Phi(t)>0,\ \forall t\in[0,s)\right\},
\end{equation*}
it  follows:
\begin{equation*}
\sup_{s\in {\cal T}_1}\left\{s: \Gamma(t)e^{-\lambda_1 t}-\Omega(t)>0,\ \forall t\in[0,s)\right\}\leq t^+\leq \sup_{s\in {\cal T}_1}\left\{s: \Gamma(t)e^{-\lambda_\alpha t}-\Omega(t)>0,\ \forall t\in[0,s)\right\},
\end{equation*}
i.e.:
$$\sup_{s\in {\cal T}_1}\left\{s: t<\frac{1}{\lambda_1}\ln\frac{\Gamma(t)}{\Omega(t)},\ \forall t\in[0,s)\right\}\leq t^+\leq \sup_{s\in {\cal T}_1}\left\{s: t<\frac{1}{\lambda_\alpha}\ln\frac{\Gamma(t)}{\Omega(t)},\ \forall t\in[0,s)\right\},$$
i.e.:
$$t^+_1\leq t^+\leq t^+_\alpha,$$
i.e. \eqref{eq:encadrement}.
    \end{proof}
\newpage

\vskip 0.2in
\bibliography{sample}

\end{document}